\pdfoutput=1

\documentclass[11pt]{article}

\usepackage[]{EMNLP2022}

\usepackage{times}
\usepackage{latexsym}
\usepackage{color}
\usepackage{latexsym}
\usepackage{graphicx}
\usepackage{amsmath}
\usepackage{booktabs}
\usepackage{multirow}
\usepackage{url}
\usepackage{amssymb}
\usepackage{amsthm}
\usepackage{enumitem}
\usepackage{makecell}
\usepackage{colortbl}

\newtheorem{theorem}{Theorem}
\newtheorem{corollary}{Corollary}
\newtheorem{assumption}{Assumption}

\newcommand{\ironman}[1]{\includegraphics[width=#1\textwidth]{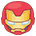}}

\newcommand{\coconut}{\includegraphics[width=0.02\textwidth]{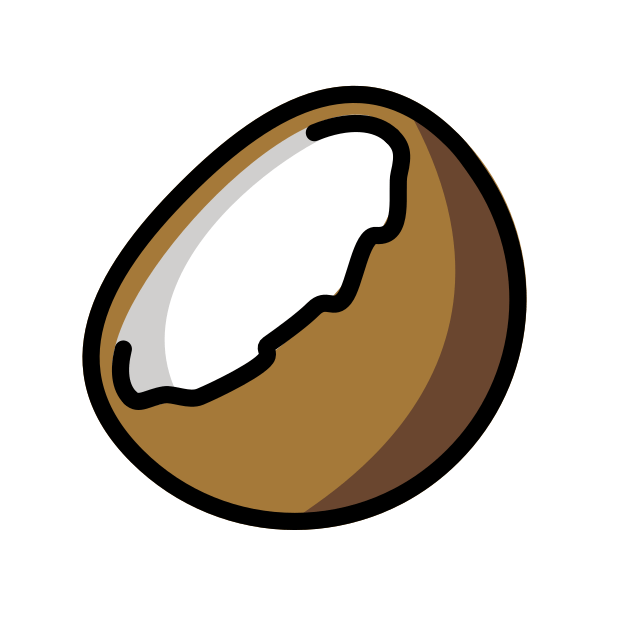}}
\newcommand{\kiwi}{\includegraphics[width=0.02\textwidth]{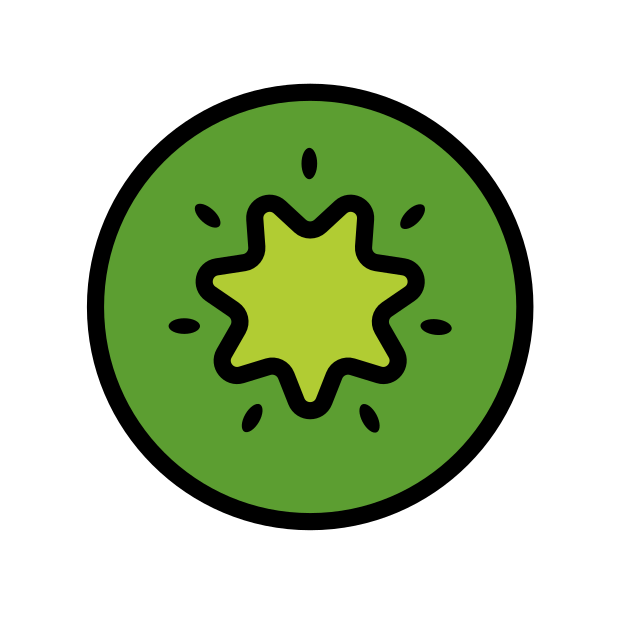}}

\usepackage[T1]{fontenc}

\usepackage[utf8]{inputenc}

\usepackage{microtype}

\usepackage{inconsolata}

\title{Sparse Teachers Can Be Dense with Knowledge}

\author{Yi Yang\textsuperscript{\coconut}\Thanks{\textsuperscript{\coconut}Yi Yang and Chen Zhang contribute equally to this work, and the order is determined alphabetically.}, Chen Zhang\textsuperscript{\coconut}, Dawei Song\textsuperscript{\kiwi}\Thanks{\textsuperscript{\kiwi}Dawei Song is the corresponding author, who is also with The Open University, UK.} \\
Beijing Institute of Technology \\
\texttt{\{yang.yi,czhang,dwsong\}@bit.edu.cn} \\}

\makeatletter
\def\thanks#1{\protected@xdef\@thanks{\@thanks
    \protect\footnotetext{#1}}}
\makeatother

\begin{document}
\maketitle
\begin{abstract}
Recent advances in distilling pretrained language models have discovered that, besides the expressiveness of knowledge, the student-friendliness should be taken into consideration to realize a truly knowledgeable teacher. Based on a pilot study, we find that over-parameterized teachers can produce expressive yet student-unfriendly knowledge and are thus limited in overall knowledgeableness. To remove the parameters that result in student-unfriendliness, we propose a sparse teacher trick under the guidance of an overall knowledgeable score for each teacher parameter. The knowledgeable score is essentially an interpolation of the expressiveness and student-friendliness scores. The aim is to ensure that the expressive parameters are retained while the student-unfriendly ones are removed. Extensive experiments on the GLUE benchmark show that the proposed sparse teachers can be dense with knowledge and lead to students with compelling performance in comparison with a series of competitive baselines.\footnote{Code is available at \url{https://github.com/GeneZC/StarK}.}
\end{abstract}

\section{Introduction}

Pretrained language models (LMs) built upon transformers~\citep{DevlinCLT19,Liu19,Raffel20} have achieved great successes. However, the appealing performance is usually accompanied with expensive computational costs and memory footprints, which can be alleviated by model compression~\citep{Ganesh21}. Knowledge distillation~\citep{Hinton15}, as a dominant method in model compression, concentrates on transferring knowledge from a teacher of large scale to a student of smaller scale. 

\begin{figure}[t]
\centering
\includegraphics[width=0.37\textwidth]{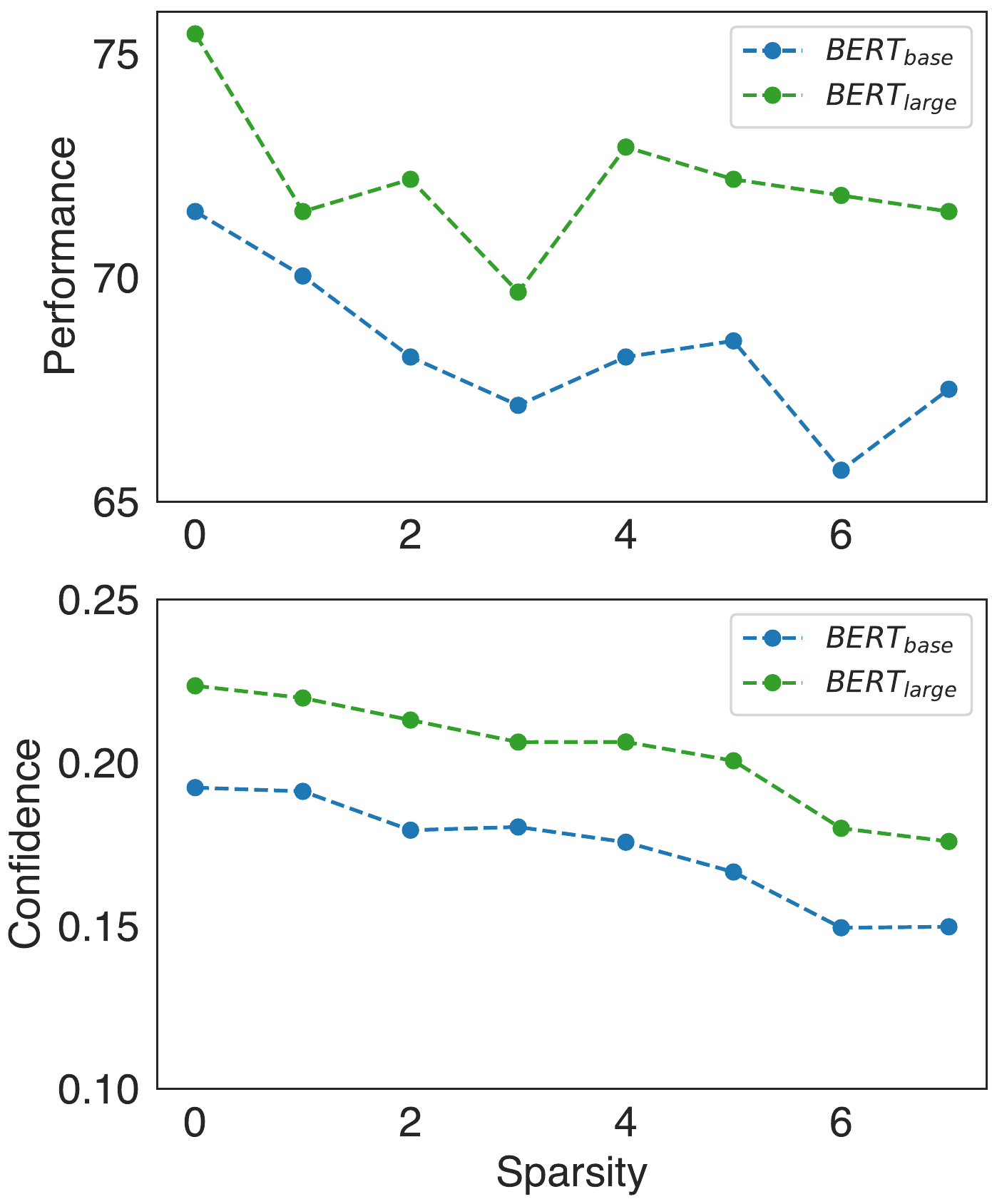}
\caption{Performance and confidence on RTE~\citep{BentivogliMDDG09} of BERT\textsubscript{base} and BERT\textsubscript{large} at small sparsity levels. Task metric and output distribution variance are used as the measures of performance and confidence, respectively. Distribution variance is comparatively equivalent to distribution negative entropy as employed in~\citet{PereyraTCKH17}. Proof of the equivalence can be found in Appendix~\ref{appendix1}.}
\label{fig1}
\end{figure}

Conventional studies~\citep{Sun19,Jiao20} mainly expect that the expressive knowledge would be well transferred, yet largely neglecting the existence of student-unfriendly knowledge. Recent attempts~\citep{Zhou22,Zhao22} are made to adapt the teacher to more student-friendly knowledge and have yielded performance gains. Based on these observations, we posit that over-parameterized LMs, on the one hand, can produce expressive knowledge due to over-parameterization, but on the other hand can also produce student-unfriendly knowledge due to over-confidence~\citep{Hinton15,PereyraTCKH17}. From a pilot study shown in Figure~\ref{fig1}, we find that LMs of large scale tend to have a good performance and high confidence, and that both performance and confidence can be degraded through randomly sparsifying a small portion of parameters.\footnote{\url{https://pytorch.org/docs/stable/generated/torch.nn.utils.prune.random_unstructured}} This indicates that some parameters resulting in student-unfriendliness can be rather removed, to improve student-friendliness of the teacher without sacrificing too much its expressiveness.

Motivated by this finding, we propose a sparse teacher trick (in short, \textsc{StarK} \ironman{0.02}) under the guidance of an overall knowledgeable score for each teacher parameter, which accords not only with the expressiveness but also the student-friendliness of the parameter by interpolation. The aim is to retain the expressive parameters while removing the student-unfriendly ones. Specifically, we introduce a three-stage procedure consisting of 1) \textit{trial distillation}, 2) \textit{parameter sparsification}, and 3) \textit{actual distillation}. The \textit{trial distillation} distills the dense teacher to the student so that a trial student is obtained. The \textit{parameter sparsification} first estimates the expressiveness score and student-friendliness score of each teacher parameter via feedbacks respectively from the teacher itself and the trial student, and then sparsifies the teacher by removing the parameters associated with adequately low interpolated knowledgeable scores. The \textit{actual distillation} distills the sparsified teacher to the student so that an actual student is obtained, where the student is initialized in the same manner as that used in \textit{trial distillation} following the commonly-used rewinding technique~\citep{FrankleC19}. 


We conduct an extensive set of experiments on the GLUE benchmark. Experimental results demonstrate that the sparse teachers can be dense with knowledge and lead to a remarkable performance of students compared with a series of competitive baselines.

\section{Background}

\subsection{BERT Architecture}

The BERT~\citep{DevlinCLT19} is composed of several stacked encoder layers of transformers~\citep{Vaswani17}. There are two blocks in every encoder layer: a multi-head self-attention block (MHA) and a feed-forward network block (FFN), with a residual connection and a normalization layer around each.

Given an $l$-length sequence of $d$-dimensional input vectors $\mathbf{X}\in\mathbb{R}^{l\times d}$, the output of the MHA block with $A$ independent heads can be represented as:
\begin{equation}\nonumber
\begin{aligned}
&{\rm MHA}(\mathbf{X})\\
&=\sum_{i=1}^{A}{\rm Attn}(\mathbf{X},\mathbf{W}_{Q}^{(i)},\mathbf{W}_{K}^{(i)},\mathbf{W}_{V}^{(i)})\mathbf{W}_O^{(i)},
\end{aligned}
\end{equation}
where the $i$-th head is parameterized by $\mathbf{W}_{Q}^{(i)}$, $\mathbf{W}_{K}^{(i)}$, $\mathbf{W}_{V}^{(i)}\in\mathbb{R}^{d\times d_{A}}$, and $\mathbf{W}_O^{(i)}\in\mathbb{R}^{d_{A}\times d}$. 
On the other hand, the output of the FFN block is:
\begin{equation}\nonumber
{\rm FFN}(\mathbf{X})= {\rm GELU}(\mathbf{X}\mathbf{W}_1)\mathbf{W}_2,
\end{equation}
where two fully-connected layers are parameterized by $\mathbf{W}_1\in\mathbb{R}^{d\times d_{I}}$ and $\mathbf{W}_2\in\mathbb{R}^{d_{I}\times d}$ respectively.

\subsection{Knowledge Distillation}

Knowledge distillation~\citep{Hinton15} aims to transfer the knowledge from a large-scale teacher to a smaller-scale student, which is originally proposed to supervise the student with the teacher logits. With its prevalence, a tremendous amount of work has been investigated to transfer various knowledge from the teacher to the student~\citep{Romero15,ZagoruykoK17,Sun19,Jiao20,Park21,LiLZXYJ20,WangW0B0020}. PKD~\citep{Sun19} introduces a patient distillation scheme where the student learns multiple intermediate layer representations and logits from the teacher. Moreover, attention distributions~\citep{Sun20,Jiao20,LiLZXYJ20,WangW0B0020} and even high-order relations~\citep{Park21} are considered to further boost the performance. 

Since a large capacity gap between the teacher and the student can lead to an inferior distillation quality, TAKD~\citep{Mirzadeh20} proposes to insert teacher assistants of possible intermediate scales between the teacher and the student so that the gap is drawn closer~\citep{zhang22}. More recently, teachers with student-friendly architectures have exactly showed the significance of student-friendliness~\citep{ParkCJKH21}. MetaKD~\citep{Zhou22} adopts meta-learning to optimize the student-friendliness of the teacher according to the student preference. DKD~\citep{Zhao22} decouples and amplifies student-friendly knowledge in contrast to others. Distinguished from these student-friendly teachers that are achieved by altering teacher scales, architectures, parameters or knowledge representations, our work, to our best knowledge, is the first one suggesting that teacher parameters can produce both student-friendly and student-unfriendly knowledge and aiming to find the sparse teacher with the best student-friendliness.

\begin{figure*}[ht]
\centering
\includegraphics[width=0.97\textwidth]{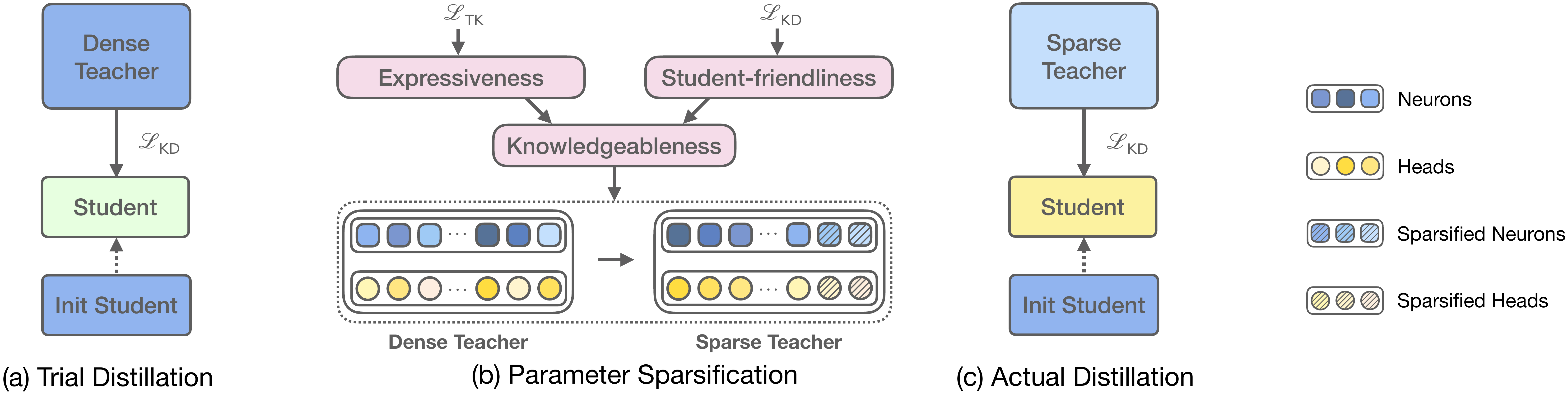}
\caption{The overview of \textsc{StarK}. \textit{trial distillation} distils a trial student from the dense teacher on a specific task. \textit{parameter sparsification} sparsifies the parameters of the dense teacher that are associated with adequately low knowledgeable scores. \textit{actual distillation} rewinds the \textit{trial distillation} by replacing the dense teacher with the proposed sparse teacher.}
\label{fig2}
\end{figure*}

\subsection{Model Pruning}

Model pruning is imposed to remove the less expressive parameters for model compression. Previous work applies either structured~\citep{Li17,LuoL17,He17,YangZWS22} or unstructured pruning~\citep{Han15,Park17,Louizos18,Lee19} to transformers. Unstructured pruning focuses on pruning parameter-level parameters based on zero-order decisions derived from magnitudes~\citep{Gordon20} or first-order decisions computed from both gradients and magnitudes~\citep{Sanh20}. In contrast, structured pruning prunes module-level parameters such like MHA heads~\citep{Michel19} and FFN layers~\citep{Prasanna20} guided by the expressive score~\citep{Michel19}. It is noteworthy that while some pruning methods leverage post-training pruning~\cite{Hou20}, others can take advantage of training-time pruning~\cite{Xia22}. Although training-time pruning can result in slightly better performance, it can consume much more time to meet a convergence. Our work mainly exploits structured pruning to obtain sparse teachers, yet also explores the use of unstructured pruning, in a post-training style.

\section{Sparse Teacher Trick}

Our trick involves three stages in the student learning procedure as shown in Figure~\ref{fig2}. First, we distil a trial student from the dense teacher on a specific task (\textit{trial distillation}). Then, we sparsify the parameters of the dense teacher that are associated with adequately low knowledgeable scores (\textit{parameter sparsification}). Finally, rewinding is applied, where the student is set to the initialization exactly used in the \textit{trial distillation} stage and is learned from the sparse teacher during (\textit{actual distillation}).

\subsection{Trial and Actual Distillations}

\textit{Trial distillation} and \textit{actual distillation} share the same distillation regime. We employ the widely-used logits distillation~\citep{Hinton15} as the distillation objective, as depicted below:
\begin{equation}\nonumber
\begin{aligned}
&\mathcal{L}_{\sf KD}=-{\rm softmax}(\mathbf{z}^{t}/\tau)\log{\rm softmax}(\mathbf{z}^{s}/\tau),\\
&\mathcal{L}_{\sf TK}=-\mathbf{y}\log\mathbf{y}^{s},\\
&\mathcal{L}=\mathcal{L}_{\sf KD}+\alpha\cdot\mathcal{L}_{\sf TK},
\end{aligned}
\end{equation}
where $\mathbf{z}^{t}$, $\mathbf{z}^{s}$ separately stand for logits of the teacher and student, and $\mathbf{y}^{s}$, $\mathbf{y}$ separately stand for prediction normalized probabilities of the student and ground-truth one-hot probabilities. Two subscripts ${\sf KD}$ and ${\sf TK}$ indicate distillation and task losses respectively. $\tau$ is a temperature controlling the smoothness the logits~\citep{Hinton15}, and $\alpha$ is a term balancing two losses. 

The \textit{trial distillation} and \textit{actual distillation} also reuse the initialization of the student for better convergence, which is known as rewinding technique~\citep{FrankleC19}. 

\subsection{Parameter Sparsification}

For \textit{parameter sparsification}, we design a knowledgeable score, which is essentially an interpolation of the already-proposed expressive score~\citep{MolchanovTKAK17} and our proposed student-friendly score, to measure the knowledgeableness of each teacher parameter. Thanks to the knowledgeable score, we can safely exclude student-unfriendly parameters without harming expressive parameters too much.

We mainly sparsify the attention heads of MHA blocks and intermediate neurons of FFN blocks in the teacher. Following the literature on structured pruning in a post-training style~\citep{Michel19,Hou20}, we attach a set of variables $\xi^{(i)}$ and $\nu$ to the attention heads and the intermediate neurons, to record the parameter sensitivities for a specific task through accumulated absolute gradients, as shown below:
\begin{equation}\nonumber
\begin{aligned}
&{\rm MHA}^{\circ}(\mathbf{X})\\
&=\sum_{i=1}^{A}\xi^{(i)}{\rm Attn}(\mathbf{X},\mathbf{W}_{Q}^{(i)},\mathbf{W}_{K}^{(i)},\mathbf{W}_{V}^{(i)})\mathbf{W}_O^{(i)},\\
&{\rm FFN}^{\circ}(\mathbf{X})= {\rm GELU}(\mathbf{X}\mathbf{W}_1){\rm diag}(\nu)\mathbf{W}_2,
\end{aligned}
\end{equation}
where $\xi^{(i)}\equiv 1$ and $\nu\equiv\mathbf{1}^{d_{I}}$. We set the values of the $\xi^{(i)}$ and $\nu$ to ones to ensure the functionalities of corresponding heads and neurons are retained. 

The implementation is mathematically equivalent to the prevalent first-order taylor expansion of the absolute variation between before and after removing a module (i.e., a head or a neuron) akin to~\citet{MolchanovTKAK17}. Take the $i$-th attention head as an example, its parameter sensitivity can be written as:
\begin{equation}\nonumber
\begin{aligned}
&\left|\frac{\partial\mathcal{L}}{\partial\xi^{(i)}}\right|=\left|\frac{\partial\mathcal{L}}{\partial\xi^{(i)}\mathbf{O}^{(i)}}\frac{\partial\xi^{(i)}\mathbf{O}^{(i)}}{\partial\xi^{(i)}}\right|=\left|\frac{\partial\mathcal{L}}{\partial\mathbf{O}^{(i)}}\mathbf{O}^{(i)}\right|\\
&\approx\left|(\mathcal{L}_{\mathbf{0}}+\frac{\partial\mathcal{L}}{\partial\mathbf{O}^{(i)}}(\mathbf{O}^{(i)}-\mathbf{0})+\mathbf{r})-\mathcal{L}_{\mathbf{0}}\right|\\
&=\left|\mathcal{L}-\mathcal{L}_{\mathbf{0}}\right|,
\end{aligned}
\end{equation}
where $\mathcal{L}$ stands for an arbitrary objective with abuse of notation, and $\mathbf{O}^{(i)}$ is utilized for $i$-th attention head output. $\mathcal{L}_{\mathbf{0}}$ actually means $\mathcal{L}|_{\mathbf{O}^{(i)}=\mathbf{0}}$, and $\mathbf{r}$ represents residuals in taylor expansion.

Note that our trick can be flexibly extended to a training-time style~\citep{Xia22} or unstructured pruning, which will be discussed in our experiments. 

\paragraph{Expressiveness.} The expressiveness of the teacher is tied to the expressiveness score. A higher expressiveness score indicates that the corresponding parameter has bigger contribution towards the performance. Concretely, the expressiveness scores of the attention heads in MHA and the intermediate neurons in FFN can be depicted as:
\begin{equation}\nonumber
\begin{aligned}
\mathbb{P}_{\sf head}^{(i)}&=\mathbb{E}_{\mathcal{D}}\left|\frac{\partial\mathcal{L}_{\sf TK}}{\partial\xi^{(i)}}\right|,\\
\mathbb{P}_{\sf neuron}&=\mathbb{E}_{\mathcal{D}}\left|\frac{\partial\mathcal{L}_{\sf TK}}{\partial {\rm diag}(\nu)}\right|,
\end{aligned}
\end{equation}
where $\mathcal{D}$ is a data distribution, and $\mathcal{L}_{\sf TK}$ is the task loss of the teacher. $\mathbb{E}$ represents expectation.

\paragraph{Student-friendliness.} Likewise, the student-friendliness of the teacher can be described as student-friendliness scores, which are approximated from distillation loss of the trial distillation. 
\begin{equation}\nonumber
\begin{aligned}
\mathbb{Q}_{\sf head}^{(i)}&=\mathbb{E}_{\mathcal{D}}\left|\frac{\partial\mathcal{L}_{\sf KD}}{\partial\xi^{(i)}}\right|,\\
\mathbb{Q}_{\sf neuron}&=\mathbb{E}_{\mathcal{D}}\left|\frac{\partial\mathcal{L}_{\sf KD}}{\partial {\rm diag}(\nu)}\right|,
\end{aligned}
\end{equation}
where $\mathcal{L}_{\sf KD}$ is the distillation loss as computed with the trial student from the \textit{trial distillation}. Accordingly, the higher the student-friendliness score is, the more friendliness the teacher offers.

Referring to~\citet{MolchanovTKAK17}, we normalize the expressiveness and student-friendliness scores with $\ell_{2}$ norm. In view that the teacher needs to balance the expressiveness and student-friendliness, we introduce a coefficient $\lambda$ to quantify the tradeoff. Therefore, the knowledgeable score can be written in an interpolated form:
\begin{equation}\nonumber
\begin{aligned}
\mathbb{I}_{\sf head}^{(i)}&= \lambda\cdot\mathbb{P}_{\sf head}^{(i)}+(1-\lambda)\cdot\mathbb{Q}_{\sf head}^{(i)},\\
\mathbb{I}_{\sf neuron}&= \lambda\cdot\mathbb{P}_{\sf neuron}+(1-\lambda)\cdot\mathbb{Q}_{\sf neuron},
\end{aligned}
\end{equation}

\textit{Parameter sparsification} sparsifies the parameters in the teacher with adequately low knowledgeable scores. The adequacy is met by enumerating diverse sparsity levels and obtaining the one leading to the best student during the \textit{actual distillation}.

\section{Experiments}

\subsection{Data \& Metrics}

\begin{table*}[ht]
\centering
\resizebox{0.77\textwidth}{!}{
\begin{tabular}{lrrcc}
\toprule
\textbf{Dataset} & \textbf{\#Train exam.} & \textbf{\#Dev exam.} & \textbf{Max. length} & \textbf{Metric} \\
\midrule
SST-2 &  67K & 0.9K & 64 & Accuracy \\
MRPC & 3.7K & 0.4K & 128 & F1 \\
STS-B & 7K & 1.5K & 128 & Spearman Correlation \\
QQP & 364K & 40K & 128 & F1 \\
MNLI-m/mm & 393K & 20K & 128 & Accuracy \\
QNLI & 105K & 5.5K & 128 & Accuracy \\
RTE & 2.5K & 0.3K & 128 & Accuracy \\
\bottomrule
\end{tabular}}
\caption{The statistics, maximum sequence lengths, and metrics of the GLUE benchmark.}
\label{tab1}
\end{table*}

We evaluate our approach on GLUE benchmark~\citep{Wang19} that contains a collection of NLU tasks, including CoLA~\citep{Warstadt19} for linguistic acceptability, SST-2~\citep{SocherPWCMNP13} for sentiment analysis, MRPC~\citep{DolanB05}, QQP\footnote{\url{https://data.quora.com/First-Quora-Dataset-Release-Question-Pairs}} and STS-B~\citep{CerDALS17} for paraphrase similarity matching, MNLI~\citep{WilliamsNB18}, QNLI~\citep{RajpurkarZLL16} and RTE~\citep{DaganGM05,haim2006,GiampiccoloMDD07,BentivogliMDDG09} for natural language inference. Note that we exclude CoLA~\citep{Warstadt19} on which general knowledge distillation methods transfer knowledge poorly~\citep{Xia22}.

Accuracy is adopted as the evaluation metric for MNLI-m, MNLI-mm, QNLI, RTE and SST-2, and F1-score is used for MRPC, QQP. The Spearman correlation is used for STS-B. We also report the Average results on development sets of all datasets. We display the statistics of GLUE in Table~\ref{tab1}.

\subsection{Implementation \& Baselines}

We conduct experiments on an Nvidia V100. AdamW~\citep{LoshchilovH19} is applied as the optimizer. We search the learning rate within \{1, 2, 3\}$\times$10\textsuperscript{-5} and the batch size within \{16, 32\}. All training procedures are carried out within 10 epochs, with an early-stopping. We empirically find that, when temperature $\tau$ is 2.0 and distillation balance $\alpha$ is 1.0, reasonable performance is attained. The optimal sparsity is searched within \{10\%, 20\%, 30\%, 40\%, 50\%, 60\%, 70\%, 80\%, 90\%\}. Knowledgeableness tradeoff $\lambda$ is set to 0.5 for acceptable performance and its impact on the performance will be discussed later.

We finetune the original BERT\footnote{\url{https://github.com/google-research/bert}} as the teacher and distil it to the student of a smaller scale initialized by dropping 2/3 layers or pruning 70\% parameters (with above-mentioned expressiveness pruning) of the teacher, which is initialized from the teacher. We first directly finetune the student as a solid baseline (FT). Then we compare our method to conventional baselines, such as KD~\cite{Hinton15}, PKD~\cite{Sun19}, CKD~\cite{Park21}, and DynaBERT~\cite{Hou20}. Further, we compare our method to student-friendly baselines, including TAKD that employs a reasonable assistant~\citep{Mirzadeh20}, MetaKD~\citep{Zhou22} that adapts the teacher with the student feedback, and DKD~\citep{Zhao22} that amplifies the student-friendly knowledge.

\subsection{Main Comparison}

\begin{table*}[ht]
\centering
\resizebox{0.99\textwidth}{!}{
\begin{tabular}{lccccccccc}
\toprule
\textbf{Method} & {\makecell[c]{\textbf{MNLI-m} \\\textbf{Acc}}} & {\makecell[c]{\textbf{MNLI-mm} \\\textbf{Acc}}} & {\makecell[c]{\textbf{MRPC} \\\textbf{F1}}} &{\makecell[c]{\textbf{QNLI} \\\textbf{Acc}}} &  {\makecell[c]{\textbf{QQP} \\\textbf{F1}}} &  {\makecell[c]{\textbf{RTE} \\\textbf{Acc}}} & {\makecell[c]{\textbf{STSB} \\\textbf{SpCorr}}} &  {\makecell[c]{\textbf{SST-2} \\\textbf{Acc}}} & \textbf{Average} \\ 
\midrule
BERT\textsubscript{base} & 84.9  & 84.9   & 91.2  & 91.7 & 88.4  &71.5  &88.3 & 93.8 & 86.8 \\
\midrule
\multicolumn{10}{c}{\textit{layer-dropped student}} \\
\midrule
FT\textsubscript{4} &   77.5  &   77.7 &   86.0 &   85.3 &   86.1 &   65.0 &   86.5 &   89.5 &   81.7 \\
KD\textsubscript{4} &   77.7 &   77.7 &   86.9 &   85.1 &   86.1 &   65.3 &   86.4 &   89.6 &   81.8 \\
PKD\textsubscript{4} &   77.7   &   77.7    &   \textbf{87.6}  &   85.0 &   86.0  &   65.3 &   86.4  &   89.9   &   82.0 \\
CKD\textsubscript{4} &   77.7 &   77.9  &   87.2  &   85.0 &   86.2  &   64.6 &   86.4  &   89.6 &   81.8 \\
\midrule
MetaKD\textsubscript{4} & \textbackslash{} & \textbackslash{}  &   85.1 & \textbackslash{} & \textbackslash{} &   63.9 &   86.5 &   89.5 &   \textbackslash{} \\
DKD\textsubscript{4} &   77.9 &   78.0 &  86.9 &   84.8 &   86.0 &   66.3 &   86.5 &   88.8 &   81.9 \\
TAKD\textsubscript{4} &   77.1 &   77.3 &   87.2 &   84.5 &   86.3 &   \textbf{67.9} &   86.7 &   89.9 &   82.1 \\
\rowcolor{green!20} \textsc{StarK}\textsubscript{4} & \textbf{78.8} & \textbf{79.0} & 87.4 & \textbf{85.7}  & \textbf{86.5} & 67.5  & \textbf{87.2} & \textbf{90.6} & \textbf{82.8} \\
\rowcolor{green!20} \quad \S & 40\% & 50\% & 50\% & 50\% & 30\% & 60\% & 40\% & 50\%  & 46\% \\
\midrule
\multicolumn{10}{c}{\textit{parameter-pruned student}} \\
\midrule
FT\textsubscript{30\%} & 82.0 & 82.6 & 88.5 & 89.5 & 87.7 & 69.0 & 87.2 & 91.9 & 84.8 \\
KD\textsubscript{30\%} & 82.5 & 82.4 & 89.1 & 89.5 & 87.8 & 69.3 & 87.0 & 91.9 & 84.9 \\
PKD\textsubscript{30\%} & 82.5 & 82.8 & \textbf{89.5} & 89.9 & \textbf{88.0} & 68.6 & 86.4 & 91.9 & 84.9 \\
DynaBERT\textsubscript{30\%} & 81.5 & 82.8 & 87.4 & 89.1 & 86.6 & 68.1 & 87.2 & 90.3 & 84.1 \\
\midrule
DKD\textsubscript{30\%} & 82.4 & 82.4 & 88.4 & 89.6 & 87.7 & \textbf{70.4} & 87.0 & 91.9 & 85.0 \\
TAKD\textsubscript{30\%} & 82.7 & 82.3 & 89.1 & 89.8 & 87.8 & 68.6 & 87.6 & 91.9 & 85.0 \\
\rowcolor{green!20} \textsc{StarK}\textsubscript{30\%} & \textbf{82.8} & \textbf{82.9} & 89.4 & \textbf{90.0} & 87.8 & 69.7 & \textbf{87.9} & \textbf{92.2} & \textbf{85.3} \\
\rowcolor{green!20} \quad \S & 30\% & 20\% & 30\% & 70\% & 40\% & 20\% & 30\% & 40\%  & 35\% \\
\bottomrule
\end{tabular}}
\caption{The results of main comparison on GLUE development set. The best results on datasets are \textbf{boldfaced}. \S~is the optimal sparsity on each dataset. *\textsubscript{4} and *\textsubscript{30\%} mean the student is initialized by dropping 2/3 layers or pruning 70\% parameters of the teacher. \textsc{StarK}\textsubscript{4} and \textsc{StarK}\textsubscript{30\%} exactly mean KD\textsubscript{4} and KD\textsubscript{30\%} w/ \textsc{StarK}. We only report MetaKD on small datasets due to limited resources, and DynaBERT without data augmentation due to unavailable augmented data.}
\label{tab2}
\end{table*}

Table~\ref{tab2} shows the main experimental results. We can observe that \textsc{StarK} has a significant performance gain by comparing \textsc{StarK} with the original KD.
Numerically, the absolute improvements brought by \textsc{StarK} are 1.0\% and 0.4\% in term of Average. This result implies that sparse teachers can be dense with knowledge. On another note, this possibly indicates a good teacher should be a modest one.
Moreover, \textsc{StarK} achieves 0.7\% and 0.3\% absolute improvements when compared to the competitive TAKD, illustrating that sparse teachers can be more expressive and student-friendly, thereby more knowledgeable to the student than teacher assistants. It seems that student-friendly baselines can only realize a comparable performance to the conventional baselines. We argue this is not the case when student-friendly baselines, say DKD, are armed with advanced distillation objectives, say PKD. Also note that the performances of MetaKD and DynaBERT are lower than those originally reported, as the original work either initialized the student from a pretrained LM of the same scale or utilized extra augmented data. 

\subsection{Analyses}

\paragraph{Knowledgeableness Tradeoff}

To investigate the impact of the tradeoff between expressiveness and student-friendliness, we conduct more experiments by varying $\lambda$ values. Figure~\ref{fig3} illustrates the performance variation along with the change of $\lambda$. The performance generally exhibits a concave curvature, which hints that the sparsification of the teacher does face with a tradeoff between expressiveness and student-friendliness, and an ideal $\lambda$ should be not too large or too small.

\begin{figure}[ht]
\centering
\includegraphics[width=0.43\textwidth]{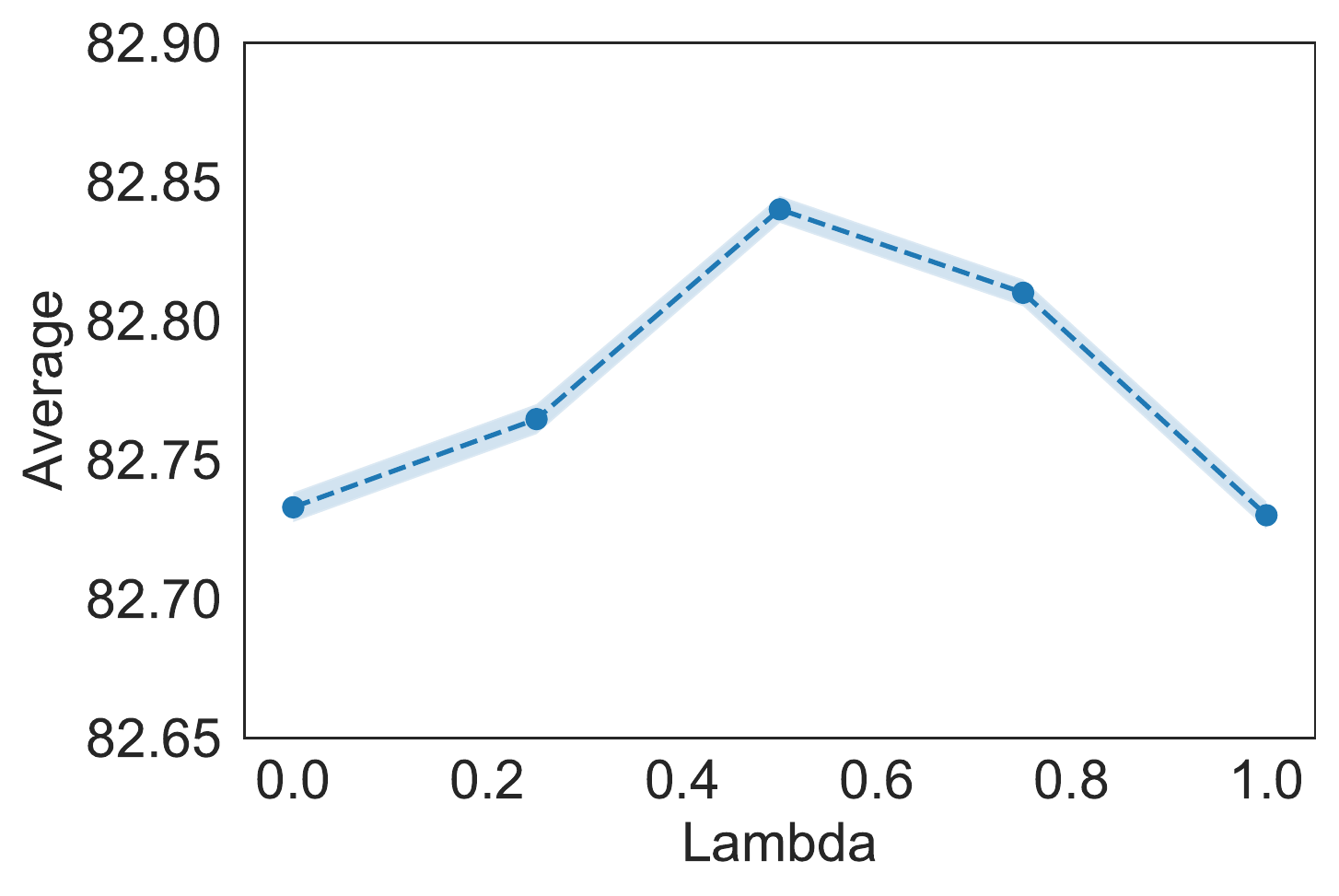}
\caption{Performance of \textsc{StarK}\textsubscript{4} with different $\lambda$.}
\label{fig3}
\end{figure}

\paragraph{Scalability}

\begin{table*}[ht]
\centering
\resizebox{0.97\textwidth}{!}{
\begin{tabular}{lcccccccccc}
\toprule
\textbf{Method} & {\makecell[c]{\textbf{MNLI-m} \\\textbf{Acc}}} & {\makecell[c]{\textbf{MNLI-mm} \\\textbf{Acc}}} & {\makecell[c]{\textbf{MRPC} \\\textbf{F1}}} &{\makecell[c]{\textbf{QNLI} \\\textbf{Acc}}} &  {\makecell[c]{\textbf{QQP} \\\textbf{Acc}}} &  {\makecell[c]{\textbf{RTE} \\\textbf{Acc}}} & {\makecell[c]{\textbf{STSB} \\\textbf{SpCorr}}} &  {\makecell[c]{\textbf{SST-2} \\\textbf{Acc}}} & \textbf{Average} \\
\midrule
BERT\textsubscript{large} & 86.6	& 86.1	& 92.3	& 92.2	& 89.0	& 75.5	& 89.9	& 93.9	& 88.2 \\
\midrule
KD\textsubscript{8} &   78.9   &   79.5    &   84.9  &   86.1 &   86.4  &   63.9 &   85.6  &   90.5   &   82.0 \\
\textsc{StarK}\textsubscript{8} & \textbf{79.4}    & \textbf{80.5}     & \textbf{85.0} & \textbf{86.3}  &   \textbf{87.0}  &   \textbf{65.7} & \textbf{88.7}   & \textbf{90.9}    & \textbf{82.9} \\ 
\quad \S & 30\% & 20\% & 90\% & 10\% & 30\% & 60\% & 20\% & 20\% & 35\% \\
\midrule
BERT\textsubscript{base} & 84.9  & 84.9   & 91.2  & 91.7 & 88.4  &71.5  &88.3 & 93.8 & 86.8 \\
\midrule
KD\textsubscript{2} &   73.2   &   72.8    &   82.9  &   78.9 &   83.5  &   \textbf{58.5} &   46.5  &   86.8   &   72.9 \\
\textsc{StarK}\textsubscript{2} & \textbf{73.9}    & \textbf{74.3}     & \textbf{83.1} & \textbf{80.4}  &   \textbf{83.8}  &   57.8 & \textbf{48.6}   & \textbf{88.1}    & \textbf{73.7} \\ 
\quad \S & 50\% & 50\% & 30\% & 50\% & 30\% & 40\% & 30\% & 40\% & 40\% \\
\bottomrule
\end{tabular}}
\caption{The results of scalability to larger teachers and smaller students.}
\label{tab3}
\end{table*}

To examine the scalability of \textsc{StarK} to larger teachers (i.e., BERT\textsubscript{large}) and smaller students (i.e., *\textsubscript{2}), where distillation methods can in fact suffer more severely from student-unfriendliness, we distil from BERT\textsubscript{large} to an eight-layer student with KD and \textsc{StarK}, and also distill from BERT\textsubscript{base} to an two-layer student. The results shown in Table~\ref{tab3} suggest that \textsc{StarK} works well on large teachers and smaller students, and the capacity gap between large teachers and small students can be drawn closer by selecting a sparse teacher. However, the eight-layer student distilled from BERT\textsubscript{large} performs only slightly better than the four-layer student distilled from BERT\textsubscript{base} even with \textsc{StarK} (see Table~\ref{tab2}).  With ~1/3 parameters,  \textsc{StarK}\textsubscript{4} can achieve 95\% performance of BERT\textsubscript{base}, and such 95\%/33\% scale-performance tradeoff is acceptable in real-world applications. In contrast, the two-layer student can only get a 85\%/17\% tradeoff, limiting its practical usage.

\begin{table}[ht]
\centering
\resizebox{0.43\textwidth}{!}{
\begin{tabular}{cc}
\toprule
\textbf{Stage} & \textbf{Train time on MNLI} \\
\midrule
\textit{trial distillation} & $\sim$2.5h \\
\textit{actual distillation} & $\sim$7h \\
\bottomrule
\end{tabular}}
\caption{The training time consumed during \textit{trial distillation} and \textit{actual distillation} stages.}
\label{tab7}
\end{table}

\paragraph{Training Efficiency}

\textsc{StarK} indeed requires more training time compared to KD due to the exhaustive search during the \textit{actual distillation} stage. However, it dose not introduce heavy compute since the search mainly involves additional distillations with sparsified teachers that are smaller than the original teacher. Table~\ref{tab7} indicates that \textit{actual distillation} consumes not that much more training time than \textit{trial distillation}. Hence, we believe the tradeoff between training time and student performance, along with training efficiency, is acceptable.

\paragraph{Pluggability}

\begin{table*}[ht]
\centering
\resizebox{0.99\textwidth}{!}{
\begin{tabular}{lccccccccc}
\toprule
\textbf{Method} & {\makecell[c]{\textbf{MNLI-m} \\\textbf{Acc}}} & {\makecell[c]{\textbf{MNLI-mm} \\\textbf{Acc}}} & {\makecell[c]{\textbf{MRPC} \\\textbf{F1}}} &{\makecell[c]{\textbf{QNLI} \\\textbf{Acc}}} &  {\makecell[c]{\textbf{QQP} \\\textbf{Acc}}} &  {\makecell[c]{\textbf{RTE} \\\textbf{Acc}}} & {\makecell[c]{\textbf{STSB} \\\textbf{SpCorr}}} &  {\makecell[c]{\textbf{SST-2} \\\textbf{Acc}}} & \textbf{Average} \\
\midrule
BERT\textsubscript{base} & 84.9  & 84.9   & 91.2  & 91.7 & 88.4  &71.5  &88.3 & 93.8 & 86.8 \\
\midrule
KD\textsubscript{4} &   77.7 &   77.7 &   86.9 &   85.1 &   86.1 &   65.3 &   86.4 &   89.6 &   81.8 \\
\quad w/ \textsc{StarK} & \textbf{78.8} & \textbf{79.0} & \textbf{87.4} & \textbf{85.7}  & \textbf{86.5} & \textbf{67.5}  & \textbf{87.2} & \textbf{90.6} & \textbf{82.8} \\ 
PKD\textsubscript{4} &   77.7   &   77.7    &   87.6  &   85.0 &   86.0  &   65.3 &   86.4  &   89.9   &   82.0 \\
\quad w/ \textsc{StarK} & \textbf{78.8}    & \textbf{79.1}     & \textbf{87.7} & \textbf{85.9}  & \textbf{86.6} & \textbf{66.8}  & \textbf{87.2}   & \textbf{90.1}    & \textbf{82.8} \\
CKD\textsubscript{4} &   77.7 &   77.9  &   87.2  &   85.0 &   86.2  &   64.6 &   86.4  &   89.6 &   81.8 \\
\quad w/ \textsc{StarK} & \textbf{78.8}    & \textbf{79.0}     & \textbf{87.6} & \textbf{86.4}  & \textbf{86.5} & \textbf{66.4}  & \textbf{87.2}   & \textbf{90.4}    & \textbf{82.8} \\
\bottomrule
\end{tabular}}
\caption{The results of pluggability to baselines.}
\label{tab4}
\end{table*}

We also show \textsc{StarK} is pluggable to any distillation methods since it is orthogonal to existing paradigms. We hence plug \textsc{StarK} to our baselines KD, PKD, and CKD to distil a four-layer student from BERT\textsubscript{base}. As in Table~\ref{tab4}, we observe that \textsc{StarK} has universal pluggability to regarded baselines, averagely improving the absolute performance by 0.9\%.

\paragraph{Unstructured Pruning}

As aforementioned, \textsc{StarK} can be flexibly applied with unstructured pruning. For unstructured pruning, we derive the expressiveness and student-friendliness scores in the same way as that used in our structured \textsc{StarK}, except the recording variables are attached to parameters rather than modules like heads. The results in Table~\ref{tab5} verify that \textsc{StarK} with unstructured pruning is slightly worse that \textsc{StarK} with structured pruning, yet it still outperforms KD. Thus, \textsc{StarK} is capable of unstructured pruning.

\begin{table*}[ht]
\centering
\resizebox{0.97\textwidth}{!}{
\begin{tabular}{lccccccccc}
\toprule
\textbf{Method} & {\makecell[c]{\textbf{MNLI-m} \\\textbf{Acc}}} & {\makecell[c]{\textbf{MNLI-mm} \\\textbf{Acc}}} & {\makecell[c]{\textbf{MRPC} \\\textbf{F1}}} &{\makecell[c]{\textbf{QNLI} \\\textbf{Acc}}} &  {\makecell[c]{\textbf{QQP} \\\textbf{Acc}}} &  {\makecell[c]{\textbf{RTE} \\\textbf{Acc}}} & {\makecell[c]{\textbf{STSB} \\\textbf{SpCorr}}} &  {\makecell[c]{\textbf{SST-2} \\\textbf{Acc}}} & \textbf{Average} \\
\midrule
BERT\textsubscript{base} & 84.9  & 84.9   & 91.2  & 91.7 & 88.4  &71.5  &88.3 & 93.8 & 86.8 \\
\midrule
KD\textsubscript{4} &   77.7 &   77.7 &   86.9 &   85.1 &   86.1 &   65.3 &   86.4 &   89.6 &   81.8 \\
\textsc{StarK}\textsubscript{4} & 78.8 & \textbf{79.0} & \textbf{87.4} & \textbf{85.7}  & 86.5 & \textbf{67.5}  & 87.2 & \textbf{90.6} & \textbf{82.8} \\ 
\textsc{StarK}\textsubscript{4}\textsuperscript{*} & \textbf{79.0} & \textbf{79.0} &  \textbf{87.4} & 85.3  & \textbf{86.8} & 66.1  & \textbf{87.3} & 89.8 & 82.6 \\
\bottomrule
\end{tabular}}
\caption{The results of compatibility with unstructured pruning. \textsuperscript{*} indicates that unstructured pruning is otherwise used.}
\label{tab5}
\end{table*}

\paragraph{Automatic \textsc{StarK}}

\begin{table*}[ht]
\centering
\resizebox{0.99\textwidth}{!}{
\begin{tabular}{lccccccccc}
\toprule
\textbf{Method} & {\makecell[c]{\textbf{MNLI-m} \\\textbf{Acc}}} & {\makecell[c]{\textbf{MNLI-mm} \\\textbf{Acc}}} & {\makecell[c]{\textbf{MRPC} \\\textbf{F1}}} &{\makecell[c]{\textbf{QNLI} \\\textbf{Acc}}} &  {\makecell[c]{\textbf{QQP} \\\textbf{Acc}}} &  {\makecell[c]{\textbf{RTE} \\\textbf{Acc}}} & {\makecell[c]{\textbf{STSB} \\\textbf{SpCorr}}} &  {\makecell[c]{\textbf{SST-2} \\\textbf{Acc}}} & \textbf{Average} \\
\midrule
\textsc{StarK}\textsubscript{4} & \textbf{78.8} & \textbf{79.0} & \textbf{87.4} & \textbf{85.7}  & \textbf{86.5} & \textbf{67.5} & \textbf{87.2} & \textbf{90.6} & \textbf{82.8} \\
\quad \S & 40\% & 50\% & 50\% & 50\% & 30\% & 60\% & 40\% & 50\%  & 46\% \\
\textsc{StarK-Auto}\textsubscript{4} & 78.1 & 79.0 &  86.6 & 85.7  & 86.0 & 67.5  & 87.2 & 90.0 & 82.6 \\ 
\quad \S & 47\% & 51\% & 35\%  & 46\% & 44\% & 42\% & 38\% & 38\% & 43\% \\
\bottomrule
\end{tabular}}
\caption{The results of \textsc{StarK-Auto}.}
\label{tab6}
\end{table*}

An issue with \textsc{StarK} is that the optimal sparsity is obtained by exhaustively enumerating all candidate sparsity levels, leading to some level of training-inefficiency. To address it, we explore an alternative algorithm to get the optimal sparsity so that \textsc{StarK} is enabled with a pursued automatic property. To this end, an attentive solution is proposed based on a surprising observation that a sparse teacher under the guidance of randomness (denoted as \textsc{StarK-Rand}\textsubscript{4}) can achieve a promising Average score of 82.5\%, whereas the scores for KD\textsubscript{4} and \textsc{StarK}\textsubscript{4} are correspondingly 81.8\% and 82.8\%. This weird phenomenon drives us to put forward a proposition.
\begin{assumption}
Both expressiveness and student-friendliness scores are densely located at their clusters, where the cluster center of student-friendliness scores owns a smaller magnitude than that of expressiveness scores.
\label{assumption1}
\end{assumption}

\begin{figure}[ht]
\centering
\includegraphics[width=0.43\textwidth]{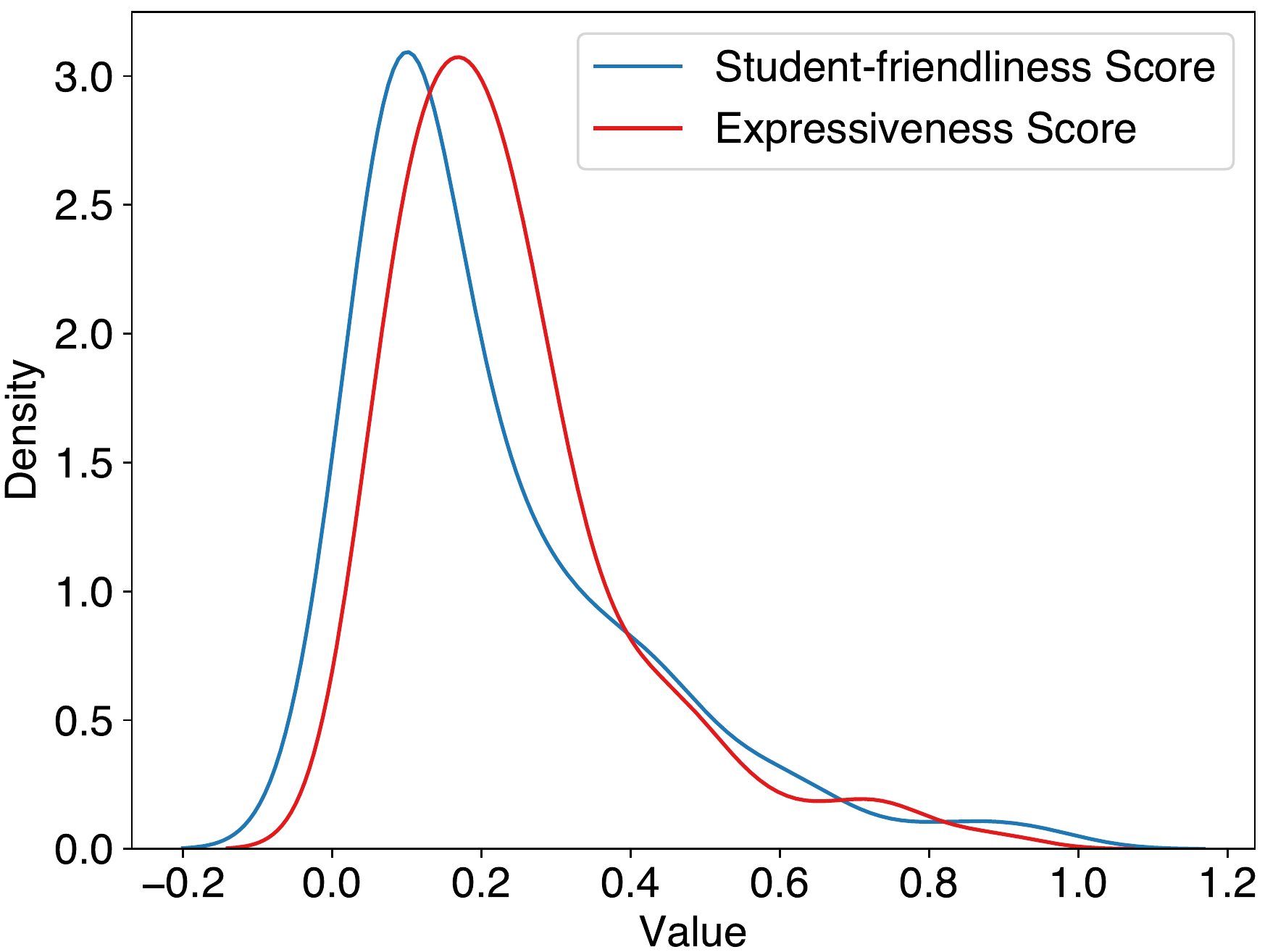}
\caption{Density of expressiveness and student-friendliness scores of BERT\textsubscript{base} attention heads finetuned on MRPC~\citep{DolanB05}. Intermediate neurons share similar characteristics, which are supplied in Appendix~\ref{appendix2}.}
\label{fig4}
\end{figure}

The assumption is intuitively verified in Figure~\ref{fig4}. When random pruning is conducted, firstly the probability of sparsifying a student-unfriendly parameter is high, and secondly the joint probability of sparsifying a student-unfriendly and inexpressive parameter is higher than that of sparsifying a student-unfriendly yet expressive parameter. Therefore, the performance of \textsc{StarK-Rand} is guaranteed with certain  probability by Assumption~\ref{assumption1}. However, we argue \textsc{StarK} is always a more robust choice than \textsc{StarK-Rand}.

\begin{figure}[ht]
\centering
\includegraphics[width=0.43\textwidth]{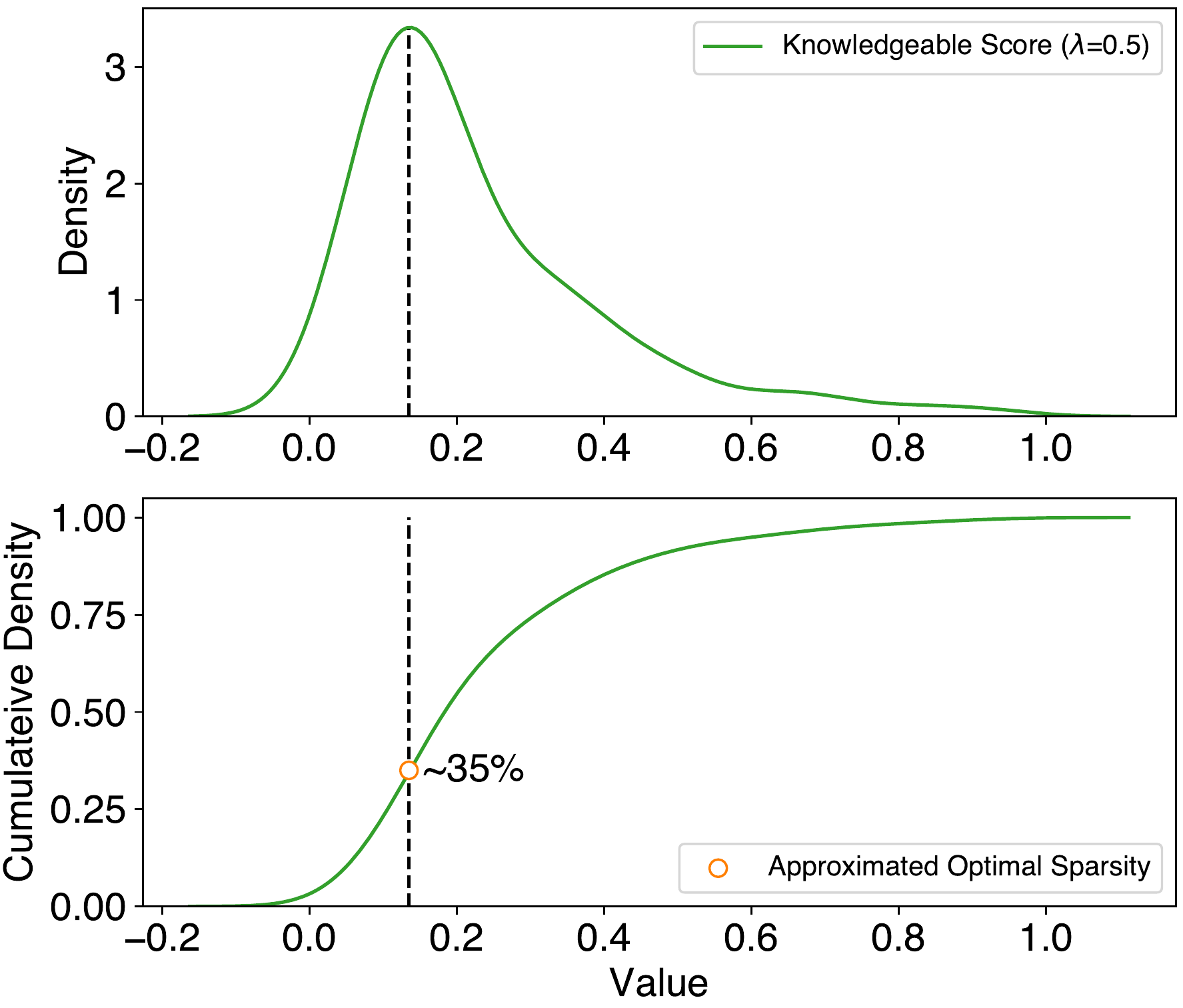}
\caption{Density and cumulative density of knowledgeableness scores of BERT\textsubscript{base} attention heads finetuned on MRPC~\citep{DolanB05}. Intermediate neurons share similar characteristics, which are supplied in Appendix~\ref{appendix2}. Full results on all tasks are supplied in Appendix~\ref{appendix3}.}
\label{fig5}
\end{figure}

A follow-up observation is that \textsc{StarK-Rand}\textsubscript{4} expects a smaller average sparsity (25\%) than \textsc{StarK}\textsubscript{4} does (46\%). It is easy to understand that the more parameters are sparsified, the lower probability above the performance guarantee will hold, though. The evident phenomenon inspires us to make another assumption.
\begin{assumption}
An optimal sparsity is positively correlated to the first density peak of a sparsification sequence.
\label{assumption2}
\end{assumption}
The assumption is illustrated in Figure~\ref{fig5}. Since \textsc{StarK-Rand} sparsifies parameters at random, it will have a small optimal sparsity as a consequence of meeting the first density peak very early. For \textsc{StarK} enjoys a sparsification sequence with only one density peak, its optimal sparsity can be automatically estimated (denoted as \textsc{StarK-Auto}) through Assumption~\ref{assumption2}. Experimental results can be found in Table~\ref{tab6}, where \textsc{StarK-Auto} approximates \textsc{StarK} in term of the Average metric. Nevertheless, we argue it is the last to use \textsc{StarK-Auto} otherwise for an extremely low practical compute as the performance can suffer a subtle drop.

\section{Conclusions}

In this paper, we validate that sparse teachers can be dense with knowledge under the guidance of our designed knowledgeable score. The idea of the sparse teacher is motivated from a pilot study, and the knowledgeable score is carefully crafted to make sure that the student-unfriendly knowledge can be reduced without hurting too much the expressive knowledge. Extensive experimental results on the GLUE benchmark support our claim to a large degree.

\section*{Limitations}
\textsc{StarK} can be further explored under two additional settings: 1) in a task-agnostic setting (e.g., MiniLM) and 2) on large LMs (e.g., BERT\textsubscript{large}). Moreover, our attentive automatic solution for \textsc{StarK} can be enhanced so that its performance can at least match the original performance.

\section*{Acknowledgements}

We thank the anonymous reviewers and chairs for their constructive suggestions. This research was supported in part by Natural Science Foundation of Beijing (grant number:  4222036) and Huawei Technologies (grant number: TC20201228005).

\bibliography{custom}

\begin{thebibliography}{48}
\expandafter\ifx\csname natexlab\endcsname\relax\def\natexlab#1{#1}\fi

\bibitem[{Bentivogli et~al.(2009)Bentivogli, Magnini, Dagan, Dang, and
  Giampiccolo}]{BentivogliMDDG09}
Luisa Bentivogli, Bernardo Magnini, Ido Dagan, Hoa~Trang Dang, and Danilo
  Giampiccolo. 2009.
\newblock \href
  {https://tac.nist.gov/publications/2009/additional.papers/RTE5\_overview.proceedings.pdf}
  {The fifth {PASCAL} recognizing textual entailment challenge}.
\newblock In \emph{TAC}.

\bibitem[{Cer et~al.(2017)Cer, Diab, Agirre, Lopez{-}Gazpio, and
  Specia}]{CerDALS17}
Daniel~M. Cer, Mona~T. Diab, Eneko Agirre, I{\~{n}}igo Lopez{-}Gazpio, and
  Lucia Specia. 2017.
\newblock \href {https://doi.org/10.18653/v1/S17-2001} {Semeval-2017 task 1:
  Semantic textual similarity multilingual and crosslingual focused
  evaluation}.
\newblock In \emph{SemEval@ACL}, pages 1--14.

\bibitem[{Dagan et~al.(2005)Dagan, Glickman, and Magnini}]{DaganGM05}
Ido Dagan, Oren Glickman, and Bernardo Magnini. 2005.
\newblock \href {https://doi.org/10.1007/11736790\_9} {The {PASCAL} recognising
  textual entailment challenge}.
\newblock In \emph{{PASCAL}, {MLCW}}, volume 3944, pages 177--190.

\bibitem[{Devlin et~al.(2019)Devlin, Chang, Lee, and Toutanova}]{DevlinCLT19}
Jacob Devlin, Ming{-}Wei Chang, Kenton Lee, and Kristina Toutanova. 2019.
\newblock \href {https://doi.org/10.18653/v1/n19-1423} {{BERT:} pre-training of
  deep bidirectional transformers for language understanding}.
\newblock In \emph{NAACL}, pages 4171--4186.

\bibitem[{Dolan and Brockett(2005)}]{DolanB05}
William~B. Dolan and Chris Brockett. 2005.
\newblock \href {https://aclanthology.org/I05-5002/} {Automatically
  constructing a corpus of sentential paraphrases}.
\newblock In \emph{IWP@IJCNLP}.

\bibitem[{Frankle and Carbin(2019)}]{FrankleC19}
Jonathan Frankle and Michael Carbin. 2019.
\newblock \href {https://openreview.net/forum?id=rJl-b3RcF7} {The lottery
  ticket hypothesis: Finding sparse, trainable neural networks}.
\newblock In \emph{ICLR}.

\bibitem[{Ganesh et~al.(2021)Ganesh, Chen, Lou, Khan, Yang, Sajjad, Nakov,
  Chen, and Winslett}]{Ganesh21}
Prakhar Ganesh, Yao Chen, Xin Lou, Mohammad~Ali Khan, Yin Yang, Hassan Sajjad,
  Preslav Nakov, Deming Chen, and Marianne Winslett. 2021.
\newblock \href {https://aclanthology.org/2021.tacl-1.63} {Compressing
  large-scale transformer-based models: A case study on {BERT}}.
\newblock \emph{TACL}, 9:1061--1080.

\bibitem[{Giampiccolo et~al.(2007)Giampiccolo, Magnini, Dagan, and
  Dolan}]{GiampiccoloMDD07}
Danilo Giampiccolo, Bernardo Magnini, Ido Dagan, and Bill Dolan. 2007.
\newblock \href {https://aclanthology.org/W07-1401/} {The third {PASCAL}
  recognizing textual entailment challenge}.
\newblock In \emph{ACL-PASCAL@ACL}, pages 1--9.

\bibitem[{Gordon et~al.(2020)Gordon, Duh, and Andrews}]{Gordon20}
Mitchell~A. Gordon, Kevin Duh, and Nicholas Andrews. 2020.
\newblock \href {https://doi.org/10.18653/v1/2020.repl4nlp-1.18} {Compressing
  {BERT:} studying the effects of weight pruning on transfer learning}.
\newblock In \emph{RepL4NLP@ACL}, pages 143--155.

\bibitem[{Haim et~al.(2006)Haim, Dagan, Dolan, Ferro, Giampiccolo, Magnini, and
  Szpektor}]{haim2006}
R~Bar Haim, Ido Dagan, Bill Dolan, Lisa Ferro, Danilo Giampiccolo, Bernardo
  Magnini, and Idan Szpektor. 2006.
\newblock The second pascal recognising textual entailment challenge.
\newblock In \emph{PASCAL}, volume~7.

\bibitem[{Han et~al.(2015)Han, Pool, Tran, and Dally}]{Han15}
Song Han, Jeff Pool, John Tran, and William~J. Dally. 2015.
\newblock \href
  {https://proceedings.neurips.cc/paper/2015/hash/ae0eb3eed39d2bcef4622b2499a05fe6-Abstract.html}
  {Learning both weights and connections for efficient neural network}.
\newblock In \emph{NeurIPS}, pages 1135--1143.

\bibitem[{He et~al.(2017)He, Zhang, and Sun}]{He17}
Yihui He, Xiangyu Zhang, and Jian Sun. 2017.
\newblock \href {https://doi.org/10.1109/ICCV.2017.155} {Channel pruning for
  accelerating very deep neural networks}.
\newblock In \emph{ICCV}, pages 1398--1406.

\bibitem[{Hinton et~al.(2015)Hinton, Vinyals, and Dean}]{Hinton15}
Geoffrey~E. Hinton, Oriol Vinyals, and Jeffrey Dean. 2015.
\newblock \href {http://arxiv.org/abs/1503.02531} {Distilling the knowledge in
  a neural network}.
\newblock \emph{arXiv}, 1503.02531.

\bibitem[{Hou et~al.(2020)Hou, Huang, Shang, Jiang, Chen, and Liu}]{Hou20}
Lu~Hou, Zhiqi Huang, Lifeng Shang, Xin Jiang, Xiao Chen, and Qun Liu. 2020.
\newblock \href
  {https://proceedings.neurips.cc/paper/2020/hash/6f5216f8d89b086c18298e043bfe48ed-Abstract.html}
  {Dynabert: Dynamic {BERT} with adaptive width and depth}.
\newblock In \emph{NeurIPS}.

\bibitem[{Jiao et~al.(2020)Jiao, Yin, Shang, Jiang, Chen, Li, Wang, and
  Liu}]{Jiao20}
Xiaoqi Jiao, Yichun Yin, Lifeng Shang, Xin Jiang, Xiao Chen, Linlin Li, Fang
  Wang, and Qun Liu. 2020.
\newblock \href {https://doi.org/10.18653/v1/2020.findings-emnlp.372}
  {Tinybert: Distilling {BERT} for natural language understanding}.
\newblock In \emph{Findings of {EMNLP}}, volume {EMNLP} 2020, pages 4163--4174.

\bibitem[{Lee et~al.(2019)Lee, Ajanthan, and Torr}]{Lee19}
Namhoon Lee, Thalaiyasingam Ajanthan, and Philip H.~S. Torr. 2019.
\newblock \href {https://openreview.net/forum?id=B1VZqjAcYX} {Snip: single-shot
  network pruning based on connection sensitivity}.
\newblock In \emph{ICLR}.

\bibitem[{Li et~al.(2017)Li, Kadav, Durdanovic, Samet, and Graf}]{Li17}
Hao Li, Asim Kadav, Igor Durdanovic, Hanan Samet, and Hans~Peter Graf. 2017.
\newblock \href {https://openreview.net/forum?id=rJqFGTslg} {Pruning filters
  for efficient convnets}.
\newblock In \emph{ICLR}.

\bibitem[{Li et~al.(2020)Li, Liu, Zhao, Xu, Yang, and Jin}]{LiLZXYJ20}
Jianquan Li, Xiaokang Liu, Honghong Zhao, Ruifeng Xu, Min Yang, and Yaohong
  Jin. 2020.
\newblock \href {https://doi.org/10.18653/v1/2020.emnlp-main.242} {{BERT-EMD:}
  many-to-many layer mapping for {BERT} compression with earth mover's
  distance}.
\newblock In \emph{EMNLP}, pages 3009--3018.

\bibitem[{Liu et~al.(2019)Liu, Ott, Goyal, Du, Joshi, Chen, Levy, Lewis,
  Zettlemoyer, and Stoyanov}]{Liu19}
Yinhan Liu, Myle Ott, Naman Goyal, Jingfei Du, Mandar Joshi, Danqi Chen, Omer
  Levy, Mike Lewis, Luke Zettlemoyer, and Veselin Stoyanov. 2019.
\newblock \href {http://arxiv.org/abs/1907.11692} {Roberta: {A} robustly
  optimized {BERT} pretraining approach}.
\newblock \emph{arXiv}, 1907.11692.

\bibitem[{Loshchilov and Hutter(2019)}]{LoshchilovH19}
Ilya Loshchilov and Frank Hutter. 2019.
\newblock \href {https://openreview.net/forum?id=Bkg6RiCqY7} {Decoupled weight
  decay regularization}.
\newblock In \emph{ICLR}.

\bibitem[{Louizos et~al.(2018)Louizos, Welling, and Kingma}]{Louizos18}
Christos Louizos, Max Welling, and Diederik~P. Kingma. 2018.
\newblock \href {https://openreview.net/forum?id=H1Y8hhg0b} {Learning sparse
  neural networks through l{\_}0 regularization}.
\newblock In \emph{ICLR}.

\bibitem[{Luo et~al.(2017)Luo, Wu, and Lin}]{LuoL17}
Jian{-}Hao Luo, Jianxin Wu, and Weiyao Lin. 2017.
\newblock \href {https://doi.org/10.1109/ICCV.2017.541} {Thinet: {A} filter
  level pruning method for deep neural network compression}.
\newblock In \emph{ICCV}, pages 5068--5076.

\bibitem[{Michel et~al.(2019)Michel, Levy, and Neubig}]{Michel19}
Paul Michel, Omer Levy, and Graham Neubig. 2019.
\newblock \href
  {https://proceedings.neurips.cc/paper/2019/hash/2c601ad9d2ff9bc8b282670cdd54f69f-Abstract.html}
  {Are sixteen heads really better than one?}
\newblock In \emph{NeurIPS}, pages 14014--14024.

\bibitem[{Mirzadeh et~al.(2020)Mirzadeh, Farajtabar, Li, Levine, Matsukawa, and
  Ghasemzadeh}]{Mirzadeh20}
Seyed{-}Iman Mirzadeh, Mehrdad Farajtabar, Ang Li, Nir Levine, Akihiro
  Matsukawa, and Hassan Ghasemzadeh. 2020.
\newblock \href {https://ojs.aaai.org/index.php/AAAI/article/view/5963}
  {Improved knowledge distillation via teacher assistant}.
\newblock In \emph{AAAI}, pages 5191--5198.

\bibitem[{Molchanov et~al.(2017)Molchanov, Tyree, Karras, Aila, and
  Kautz}]{MolchanovTKAK17}
Pavlo Molchanov, Stephen Tyree, Tero Karras, Timo Aila, and Jan Kautz. 2017.
\newblock \href {https://openreview.net/forum?id=SJGCiw5gl} {Pruning
  convolutional neural networks for resource efficient inference}.
\newblock In \emph{ICLR}.

\bibitem[{Park et~al.(2021{\natexlab{a}})Park, Cha, Jeong, Kim, and
  Han}]{ParkCJKH21}
Dae~Young Park, Moon{-}Hyun Cha, Changwook Jeong, Daesin Kim, and Bohyung Han.
  2021{\natexlab{a}}.
\newblock \href
  {https://proceedings.neurips.cc/paper/2021/hash/6e7d2da6d3953058db75714ac400b584-Abstract.html}
  {Learning student-friendly teacher networks for knowledge distillation}.
\newblock In \emph{NeurIPS}, pages 13292--13303.

\bibitem[{Park et~al.(2021{\natexlab{b}})Park, Kim, and Yang}]{Park21}
Geondo Park, Gyeongman Kim, and Eunho Yang. 2021{\natexlab{b}}.
\newblock \href {https://doi.org/10.18653/v1/2021.emnlp-main.30} {Distilling
  linguistic context for language model compression}.
\newblock In \emph{EMNLP}, pages 364--378.

\bibitem[{Park et~al.(2017)Park, Li, Wen, Tang, Li, Chen, and Dubey}]{Park17}
Jongsoo Park, Sheng~R. Li, Wei Wen, Ping Tak~Peter Tang, Hai Li, Yiran Chen,
  and Pradeep Dubey. 2017.
\newblock \href {https://openreview.net/forum?id=rJPcZ3txx} {Faster cnns with
  direct sparse convolutions and guided pruning}.
\newblock In \emph{ICLR}.

\bibitem[{Pereyra et~al.(2017)Pereyra, Tucker, Chorowski, Kaiser, and
  Hinton}]{PereyraTCKH17}
Gabriel Pereyra, George Tucker, Jan Chorowski, Lukasz Kaiser, and Geoffrey~E.
  Hinton. 2017.
\newblock \href {https://openreview.net/forum?id=HyhbYrGYe} {Regularizing
  neural networks by penalizing confident output distributions}.
\newblock In \emph{ICLR}.

\bibitem[{Prasanna et~al.(2020)Prasanna, Rogers, and Rumshisky}]{Prasanna20}
Sai Prasanna, Anna Rogers, and Anna Rumshisky. 2020.
\newblock \href {https://doi.org/10.18653/v1/2020.emnlp-main.259} {When {BERT}
  plays the lottery, all tickets are winning}.
\newblock In \emph{EMNLP}, pages 3208--3229.

\bibitem[{Raffel et~al.(2020)Raffel, Shazeer, Roberts, Lee, Narang, Matena,
  Zhou, Li, and Liu}]{Raffel20}
Colin Raffel, Noam Shazeer, Adam Roberts, Katherine Lee, Sharan Narang, Michael
  Matena, Yanqi Zhou, Wei Li, and Peter~J. Liu. 2020.
\newblock \href {http://jmlr.org/papers/v21/20-074.html} {Exploring the limits
  of transfer learning with a unified text-to-text transformer}.
\newblock \emph{JMLR}, 21:140:1--140:67.

\bibitem[{Rajpurkar et~al.(2016)Rajpurkar, Zhang, Lopyrev, and
  Liang}]{RajpurkarZLL16}
Pranav Rajpurkar, Jian Zhang, Konstantin Lopyrev, and Percy Liang. 2016.
\newblock \href {https://doi.org/10.18653/v1/d16-1264} {Squad: 100, 000+
  questions for machine comprehension of text}.
\newblock In \emph{EMNLP}, pages 2383--2392.

\bibitem[{Romero et~al.(2015)Romero, Ballas, Kahou, Chassang, Gatta, and
  Bengio}]{Romero15}
Adriana Romero, Nicolas Ballas, Samira~Ebrahimi Kahou, Antoine Chassang, Carlo
  Gatta, and Yoshua Bengio. 2015.
\newblock \href {http://arxiv.org/abs/1412.6550} {Fitnets: Hints for thin deep
  nets}.
\newblock In \emph{ICLR}.

\bibitem[{Sanh et~al.(2020)Sanh, Wolf, and Rush}]{Sanh20}
Victor Sanh, Thomas Wolf, and Alexander~M. Rush. 2020.
\newblock \href
  {https://proceedings.neurips.cc/paper/2020/hash/eae15aabaa768ae4a5993a8a4f4fa6e4-Abstract.html}
  {Movement pruning: Adaptive sparsity by fine-tuning}.
\newblock In \emph{NeurIPS}.

\bibitem[{Socher et~al.(2013)Socher, Perelygin, Wu, Chuang, Manning, Ng, and
  Potts}]{SocherPWCMNP13}
Richard Socher, Alex Perelygin, Jean Wu, Jason Chuang, Christopher~D. Manning,
  Andrew~Y. Ng, and Christopher Potts. 2013.
\newblock \href {https://aclanthology.org/D13-1170/} {Recursive deep models for
  semantic compositionality over a sentiment treebank}.
\newblock In \emph{EMNLP}, pages 1631--1642.

\bibitem[{Sun et~al.(2019)Sun, Cheng, Gan, and Liu}]{Sun19}
Siqi Sun, Yu~Cheng, Zhe Gan, and Jingjing Liu. 2019.
\newblock \href {https://doi.org/10.18653/v1/D19-1441} {Patient knowledge
  distillation for {BERT} model compression}.
\newblock In \emph{EMNLP}, pages 4322--4331.

\bibitem[{Sun et~al.(2020)Sun, Yu, Song, Liu, Yang, and Zhou}]{Sun20}
Zhiqing Sun, Hongkun Yu, Xiaodan Song, Renjie Liu, Yiming Yang, and Denny Zhou.
  2020.
\newblock \href {https://doi.org/10.18653/v1/2020.acl-main.195} {Mobilebert: a
  compact task-agnostic {BERT} for resource-limited devices}.
\newblock In \emph{ACL}, pages 2158--2170.

\bibitem[{Vaswani et~al.(2017)Vaswani, Shazeer, Parmar, Uszkoreit, Jones,
  Gomez, Kaiser, and Polosukhin}]{Vaswani17}
Ashish Vaswani, Noam Shazeer, Niki Parmar, Jakob Uszkoreit, Llion Jones,
  Aidan~N. Gomez, Lukasz Kaiser, and Illia Polosukhin. 2017.
\newblock \href
  {https://proceedings.neurips.cc/paper/2017/hash/3f5ee243547dee91fbd053c1c4a845aa-Abstract.html}
  {Attention is all you need}.
\newblock In \emph{NeurIPS}, pages 5998--6008.

\bibitem[{Wang et~al.(2019)Wang, Singh, Michael, Hill, Levy, and
  Bowman}]{Wang19}
Alex Wang, Amanpreet Singh, Julian Michael, Felix Hill, Omer Levy, and
  Samuel~R. Bowman. 2019.
\newblock \href {https://openreview.net/forum?id=rJ4km2R5t7} {{GLUE:} {A}
  multi-task benchmark and analysis platform for natural language
  understanding}.
\newblock In \emph{ICLR}.

\bibitem[{Wang et~al.(2020)Wang, Wei, Dong, Bao, Yang, and Zhou}]{WangW0B0020}
Wenhui Wang, Furu Wei, Li~Dong, Hangbo Bao, Nan Yang, and Ming Zhou. 2020.
\newblock \href
  {https://proceedings.neurips.cc/paper/2020/hash/3f5ee243547dee91fbd053c1c4a845aa-Abstract.html}
  {Minilm: Deep self-attention distillation for task-agnostic compression of
  pre-trained transformers}.
\newblock In \emph{NeurIPS}.

\bibitem[{Warstadt et~al.(2019)Warstadt, Singh, and Bowman}]{Warstadt19}
Alex Warstadt, Amanpreet Singh, and Samuel~R. Bowman. 2019.
\newblock \href {https://transacl.org/ojs/index.php/tacl/article/view/1710}
  {Neural network acceptability judgments}.
\newblock \emph{TACL}, 7:625--641.

\bibitem[{Williams et~al.(2018)Williams, Nangia, and Bowman}]{WilliamsNB18}
Adina Williams, Nikita Nangia, and Samuel~R. Bowman. 2018.
\newblock \href {https://doi.org/10.18653/v1/n18-1101} {A broad-coverage
  challenge corpus for sentence understanding through inference}.
\newblock In \emph{NAACL-HLT}, pages 1112--1122.

\bibitem[{Xia et~al.(2022)Xia, Zhong, and Chen}]{Xia22}
Mengzhou Xia, Zexuan Zhong, and Danqi Chen. 2022.
\newblock \href {https://aclanthology.org/2022.acl-long.107} {Structured
  pruning learns compact and accurate models}.
\newblock In \emph{ACL}, pages 1513--1528.

\bibitem[{Yang et~al.(2022)Yang, Zhang, Wang, and Song}]{YangZWS22}
Yi~Yang, Chen Zhang, Benyou Wang, and Dawei Song. 2022.
\newblock \href {https://doi.org/10.1007/978-3-031-17120-8\_12} {Doge tickets:
  Uncovering domain-general language models by playing lottery tickets}.
\newblock In \emph{NLPCC}, volume 13551 of \emph{Lecture Notes in Computer
  Science}, pages 144--156.

\bibitem[{Zagoruyko and Komodakis(2017)}]{ZagoruykoK17}
Sergey Zagoruyko and Nikos Komodakis. 2017.
\newblock \href {https://openreview.net/forum?id=Sks9\_ajex} {Paying more
  attention to attention: Improving the performance of convolutional neural
  networks via attention transfer}.
\newblock In \emph{ICLR}.

\bibitem[{Zhang et~al.(2022)Zhang, Yang, Wang, Liu, Wang, Wu, and
  Song}]{zhang22}
Chen Zhang, Yang Yang, Qifan Wang, Jiahao Liu, Jingang Wang, Wei Wu, and Dawei
  Song. 2022.
\newblock \href {https://doi.org/10.48550/arXiv.2205.14570} {Autodisc:
  Automatic distillation schedule for large language model compression}.
\newblock \emph{arXiv}, 2205.14570.

\bibitem[{Zhao et~al.(2022)Zhao, Cui, Song, Qiu, and Liang}]{Zhao22}
Borui Zhao, Quan Cui, Renjie Song, Yiyu Qiu, and Jiajun Liang. 2022.
\newblock \href {https://doi.org/10.48550/arXiv.2203.08679} {Decoupled
  knowledge distillation}.
\newblock \emph{arXiv}, 2203.08679.

\bibitem[{Zhou et~al.(2022)Zhou, Xu, and McAuley}]{Zhou22}
Wangchunshu Zhou, Canwen Xu, and Julian~J. McAuley. 2022.
\newblock \href {https://aclanthology.org/2022.acl-long.485} {{BERT} learns to
  teach: Knowledge distillation with meta learning}.
\newblock In \emph{ACL}, pages 7037--7049.

\end{thebibliography}
\bibliographystyle{acl_natbib}

\appendix

\section{Comparative Equivalence of Distribution Variance and Negative Entropy}
\label{appendix1}

\begin{theorem}
For any two distributions $\mathbf{y}$ and $\mathbf{y}^{\prime}$, the negative entropy difference between them can be approximated by their variance difference.
\end{theorem}

\begin{proof}
\begin{equation}\nonumber
\begin{aligned}
&-\mathcal{H}(\mathbf{y})-(-\mathcal{H}(\mathbf{y}^{\prime}))\\
&=\sum_{i}\mathbf{y}_{i}\log\mathbf{y}_{i}-\sum_{i}\mathbf{y}^{\prime}_{i}\log\mathbf{y}^{\prime}_{i}\\
&\approx\sum_{i}(\mathbf{y}_{i}-1)+\frac{1}{2}(\mathbf{y}_{i}-1)^{2}+\mathbf{r}\\
&\qquad-\sum_{i}(\mathbf{y}_{i}^{\prime}-1)+\frac{1}{2}(\mathbf{y}_{i}^{\prime}-1)^{2}+\mathbf{r}\\
&=\sum_{i}-(\mathbf{y}_{i}-\mathbf{y}_{i}^{\prime})+\frac{1}{2}(\mathbf{y}_{i}^{2}-\mathbf{y}_{i}^{\prime 2})\\
&=\sum_{i}\frac{1}{2}((\mathbf{y}_{i}-\bar{\mathbf{y}})^{2}-(\mathbf{y}_{i}^{\prime}-\bar{\mathbf{y}}^{\prime})^{2})\\
&\propto\sum_{i}(\mathbf{y}_{i}-\bar{\mathbf{y}})^{2}-\sum_{i}(\mathbf{y}_{i}^{\prime}-\bar{\mathbf{y}}^{\prime})^{2}\\
&=\mathcal{V}(\mathbf{y})-\mathcal{V}(\mathbf{y}^{\prime}).
\end{aligned}
\end{equation}
\end{proof}

\begin{corollary}
Distribution variance, when taken as the measure of confidence, is comparatively equivalent to distribution negative entropy.
\end{corollary}

\section{Density of Scores of BERT\textsubscript{base} Intermediate Neurons on MRPC}
\label{appendix2}

\begin{figure}[ht]
\centering
\includegraphics[width=0.43\textwidth]{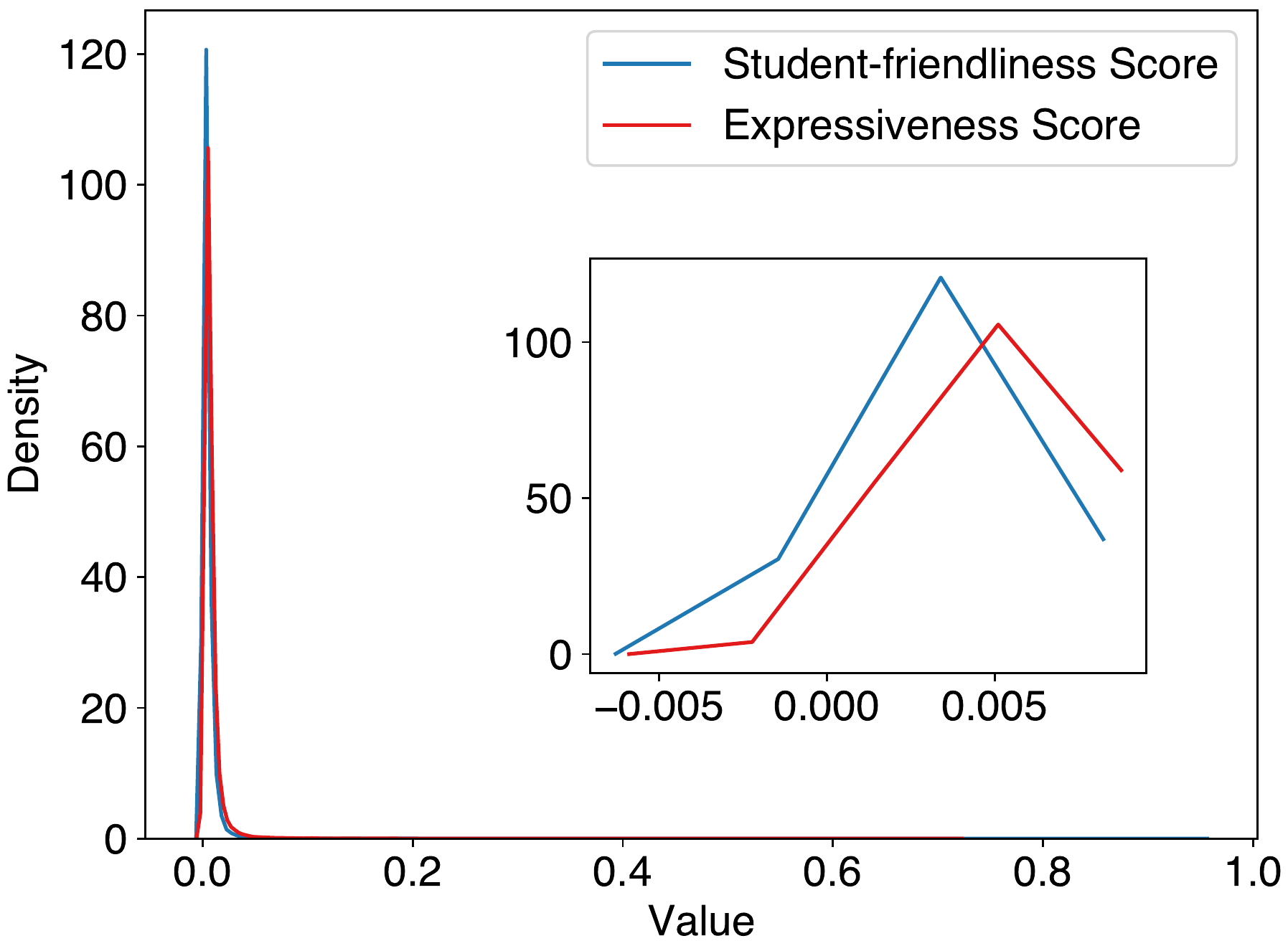}
\caption{Density of expressiveness and student-friendliness scores of BERT\textsubscript{base} intermediate neurons finetuned on MRPC~\citep{DolanB05}. }
\label{fig6}
\end{figure}

\begin{figure}[ht]
\centering
\includegraphics[width=0.43\textwidth]{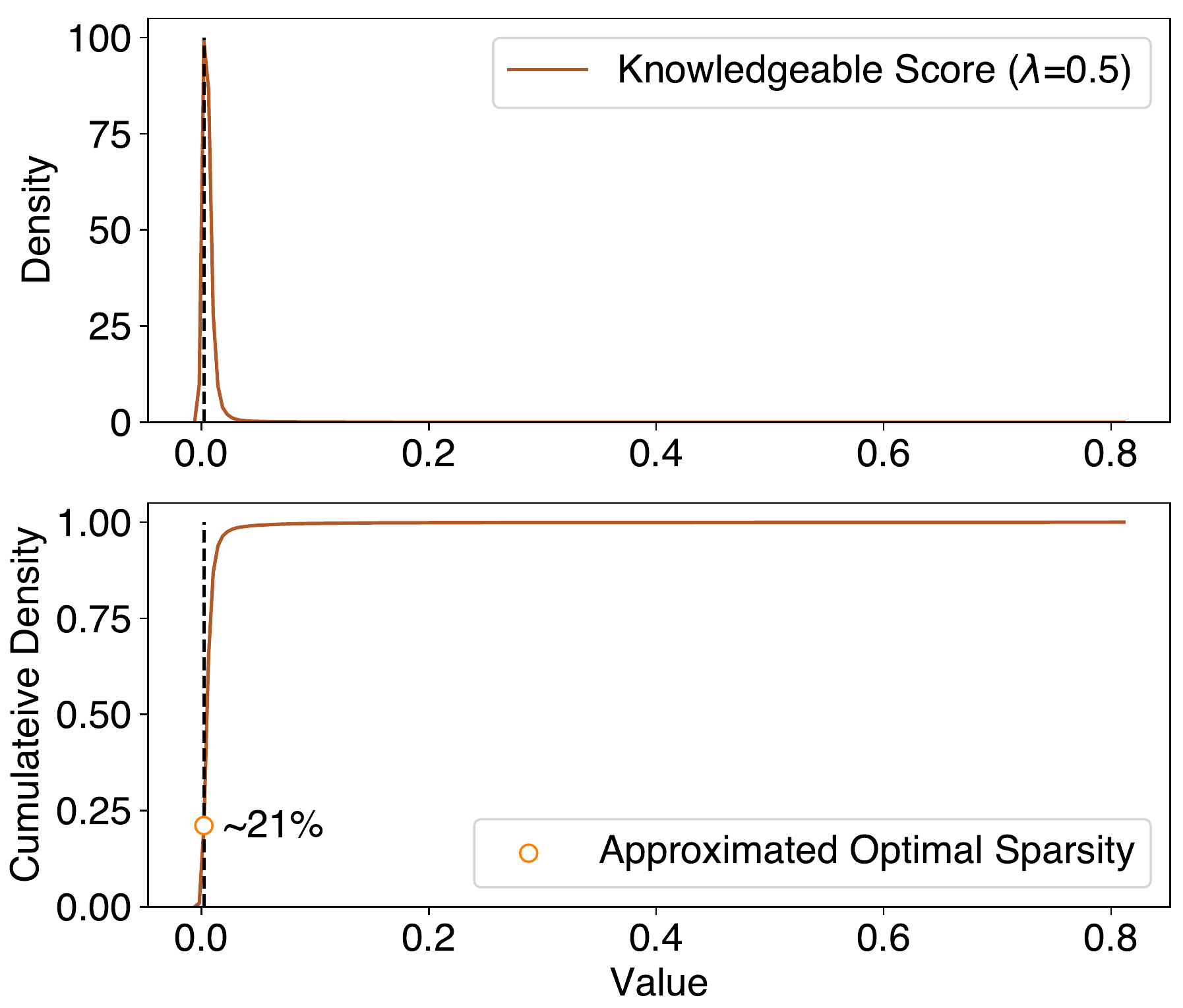}
\caption{Density and cumulative density of knowledgeableness scores of BERT\textsubscript{base} intermediate neurons finetuned on MRPC~\citep{DolanB05}.}
\label{fig7}
\end{figure}

\section{Density of Scores of BERT\textsubscript{base} Attention Heads on GLUE}
\label{appendix3}

\begin{figure*}[ht]
\centering
\includegraphics[width=0.82\textwidth]{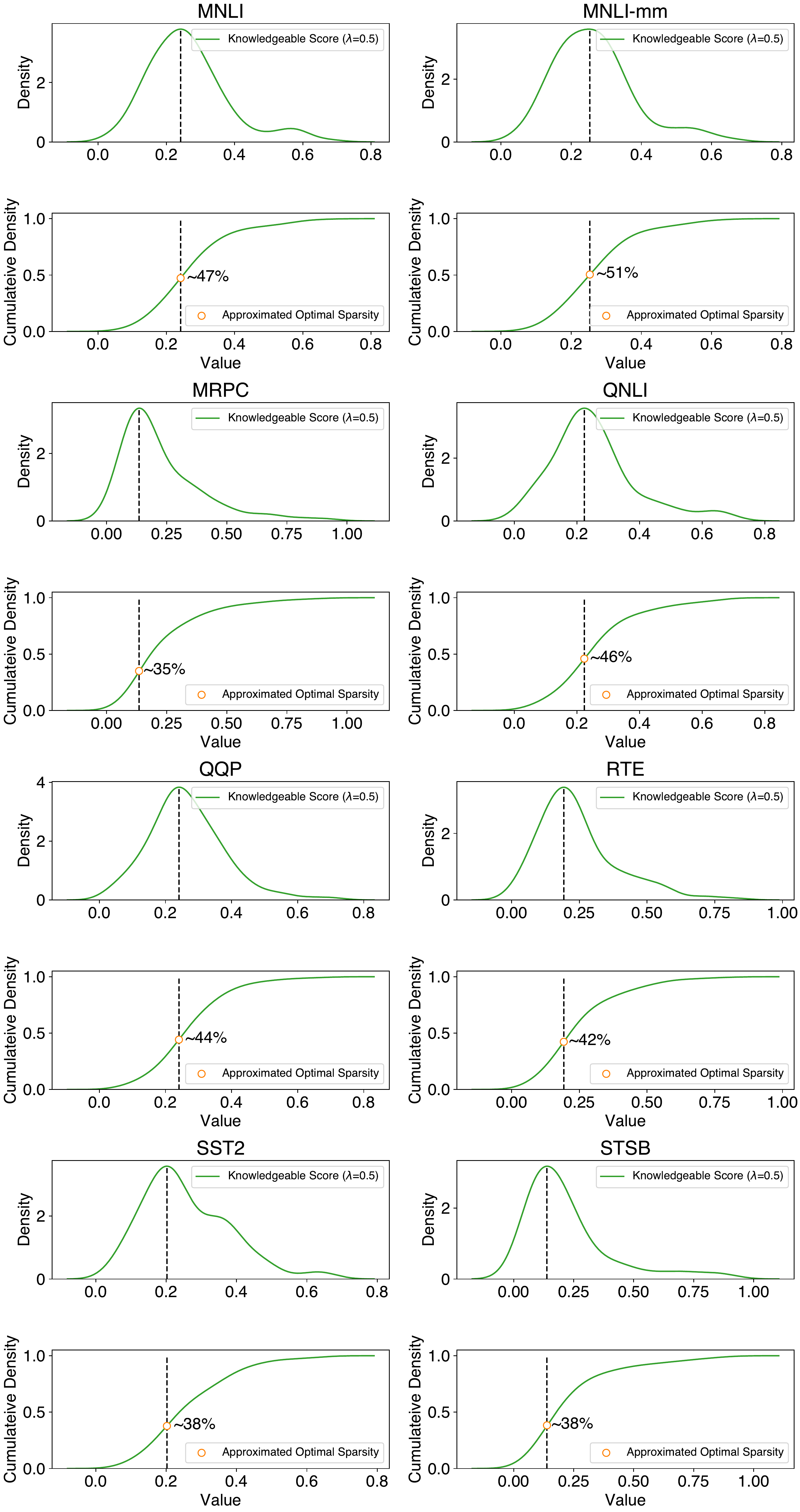}
\caption{Density and cumulative density of knowledgeableness scores of BERT\textsubscript{base} attention heads finetuned on GLUE.}
\label{fig8}
\end{figure*}

\end{document}